\documentclass[letterpaper, 10 pt, conference]{IEEEtran}
	\IEEEoverridecommandlockouts
	\usepackage{cite}
	\usepackage{amssymb,amsfonts}
        \usepackage{tikz}
        \usetikzlibrary{patterns.meta}
        \usepackage{caption}
        \usepackage{pgfplots}
        \pgfplotsset{compat=1.18}

	\usepackage{graphicx}
	\usepackage{textcomp}
        \usepackage{breqn}
	\usepackage{amsthm}
	\newtheorem{theorem}{Theorem}
	\newtheorem{lemma}{Lemma}

	\usepackage{graphicx}
        \usepackage{subcaption}
	\theoremstyle{definition}

	\theoremstyle{remark}
	\newtheorem*{remark}{Remark}

        \usepackage{algorithm}
        \usepackage{algpseudocode}
        
        \newtheorem{assumption}{Assumption}
        
\newcommand{\BEAS}{\begin{eqnarray*}}
	\newcommand{\EEAS}{\end{eqnarray*}}
\newcommand{\BEA}{\begin{eqnarray}}
	\newcommand{\EEA}{\end{eqnarray}}
\newcommand{\BEQ}{\begin{equation}}
	\newcommand{\EEQ}{\end{equation}}
\newcommand{\BIT}{\begin{itemize}}
	\newcommand{\EIT}{\end{itemize}}

\newcommand{\eg}{{\it e.g.}}
\newcommand{\ie}{{\it i.e.}}

\newcommand{\Prob}{\mathop{\mathbb{P}}}
\newcommand{\Expect}{\mathop{\mathbb{E}}}

\newcommand{\argmax}{\mathop{\rm argmax}}

\newcounter{exno}
\newlength{\exlabelwidth}

{\clearpage\section*{Exercises}%
	\addcontentsline{toc}{section}{Exercises}%
	\markright{Exercises}%
	\begin{small}\addtolength{\baselineskip}{-1.0pt}%
		\begin{list}{\bfseries\textsf{\thechapter.\arabic{exno}}}%
			{\usecounter{exno}%
				\settowidth{\exlabelwidth}{\bfseries\textsf{\thechapter.99}}%
				\setlength{\labelwidth}{\exlabelwidth}%
				\setlength{\labelsep}{.5em}%
				\setlength{\leftmargin}{0mm}%
				\setlength{\rightmargin}{0mm}}} %
		{\end{list}\clearpage%
	\end{small}\addtolength{\baselineskip}{1.0pt}}   

\makeatletter
\long\def\@makecaption#1#2{
	\vskip 9pt 
	\begin{small}
		\setbox\@tempboxa\hbox{{\sffamily\bfseries #1} #2}
		\ifdim \wd\@tempboxa > 0.85\textwidth
		\begin{center}
			\begin{minipage}[t]{0.85\textwidth}
				\addtolength{\baselineskip}{-0.95pt}
				{\sffamily\bfseries #1} #2 \par
				\addtolength{\baselineskip}{0.95pt}
			\end{minipage}
		\end{center}
		\else 
		\hbox to\hsize{\hfil\box\@tempboxa\hfil}  
		\fi
	\end{small}\par
}
\makeatother

\newcounter{oursection}

\newcounter{lecture}

\newcommand{\calC}{{\mathcal C}}
\newcommand{\calD}{{\mathcal D}}
\newcommand{\calF}{{\mathcal F}}

\newcommand{\reals}{{\mathbb{R}}}

	\usepackage{xcolor}
	\def\BibTeX{{\rm B\kern-.05em{\sc i\kern-.025em b}\kern-.08em
	    T\kern-.1667em\lower.7ex\hbox{E}\kern-.125emX}}
         \definecolor{DarkGreen}{RGB}{0,100,0}

\newcounter{edremcounter}
\usepackage[hidelinks]{hyperref}
\definecolor{darkred}{RGB}{150,0,0}
\definecolor{darkgreen}{RGB}{0,150,0}
\definecolor{darkblue}{RGB}{0,0,150}
\hypersetup{colorlinks=true, linkcolor=red, citecolor=blue, urlcolor=darkblue}

\begin{document}
		
		\title{Cooperative Multi-Agent Constrained Stochastic Linear Bandits
		}

            \author{
    Amirhossein Afsharrad$^{1,2}$$^*$,
    Parisa Oftadeh$^3$$^*$,
    Ahmadreza Moradipari$^4$$^*$,
    Sanjay Lall$^2$
    \thanks{$^*$Equal contribution. $^1$Aktus AI, 
    \texttt{amir@aktus.ai}.
    $^2$Stanford University, \texttt{\{afsharrad,lall\}@stanford.edu}. $^3$University of California, Santa Cruz, \texttt{poftadeh@ucsc.edu}.
    $^4$University of California, Santa Barbara, \texttt{ahmadreza_moradipari@ucsb.edu}.
    This work was supported by NSF under ECCS CPS project number 2125511.
    }
}
	
		
		\maketitle
  
\begin{abstract}
In this study, we explore a collaborative multi-agent stochastic linear bandit setting involving a network of \(N\) agents that communicate locally to minimize their collective regret while keeping their expected cost under a specified threshold \(\tau\). Each agent encounters a distinct linear bandit problem characterized by its own reward and cost parameters, i.e., local parameters. The goal of the agents is to determine the best overall action corresponding to the average of these parameters, or so-called global parameters. In each round, an agent is randomly chosen to select an action based on its current knowledge of the system. This chosen action is then executed by all agents, then they observe their individual rewards and costs. We propose a safe distributed upper confidence bound algorithm, so called \textit{MA-OPLB}, and establish a high probability bound on its \(T\)-round regret. MA-OPLB utilizes an accelerated consensus method, where agents can compute an estimate of the average rewards and costs across the network by communicating the proper information with their neighbors. We show that our regret bound is of order $ \mathcal{O}\left(\frac{d}{\tau-c_0}\frac{\log(NT)^2}{\sqrt{N}}\sqrt{\frac{T}{\log(1/|\lambda_2|)}}\right)$, where $\lambda_2$ is the second largest (in absolute value) eigenvalue of the communication matrix, and $\tau-c_0$ is the known cost gap of a feasible action. We also experimentally show the performance of our proposed algorithm in different network structures. 
\end{abstract}

\section{Introduction}

Stochastic linear bandits have been widely researched in decision-making scenarios with a linear framework, such as recommendation systems or path routing \cite{DBLP:journals/corr/abs-1204-5721, abbasi}. In these problems, at each time step, an agent selects an action and receives a corresponding random reward, which has an expected value that depends linearly on the context of the action. The agent's objective is to maximize the total reward over \(T\) rounds. In this work, we study the constrained stochastic linear bandit problem in a multi-agent context, where a group of \(N\) agents work together locally within a network to collectively maximize their rewards while ensuring that the expected total cost remains below a specified threshold, \(\tau\). In particular, each agent is interacting with a local constrained linear bandit problem with unknown reward and cost parameters. However, the goal of the agent is to collaboratively select the optimal action with respect to the global network parameters. Each agent deals with its unique constrained linear bandit problem, meaning that the reward and cost parameters vary across agents. The network's overall reward and cost parameters are calculated as the averages of all the agents' parameters. To minimize excessive communication, we establish two key assumptions: agents can only share information with their immediate neighbors within the network, and each communication step causes an increase in regret. Under these conditions, we provide a regret bound that depends on the spectral gap of the structure matrix \((1 - |\lambda_2|)\) and the cost gap of a known feasible action. 

%

As a motivational example, consider a fashion brand that utilizes a network of influencers across various locations, where these influencers act as agents within our model. The brand’s goal is to maximize its global welfare while keeping the overall cost below a known threshold (e.g., a budget constraint). In this scenario, each influencer interacts with their own target customers, which is represented as their local constrained linear bandit problem. Agents collaborate with their neighbors, collectively steering towards the optimal action that not only aligns with the network's true parameters but also adheres to the global budget constraint (i.e., finding the most suitable and cost-effective product for the entire network).


\textbf{Linear Bandits.} In stochastic linear bandit (LB) problems, actions are represented by feature vectors, and the expected reward for each action is a linear function of its feature vector. Two well-known algorithms for LB are linear UCB (LinUCB) and linear Thompson Sampling (LinTS). The study in \cite{abbasi} established a regret bound of order \(O(\sqrt{T \log T})\) for LinUCB, while the works in \cite{pmlr-v54-abeille17a},  demonstrated a regret bound of order \(O(\sqrt{T(\log T)^{3/2}})\) for LinTS in a frequentist setting, where the unknown reward parameter \(\theta_*\) is fixed.

\textbf{Constrained Bandits.} In constrained bandit (CB) problems, the goal is to maximize the cumulative reward while ensuring the costs stay below a specified threshold \cite{moradipari2021safe, varma2023stochastic, hutchinson2024safe }. In our setting, we focus on a specific algorithm called Optimistic Pessimistic Linear Bandit (OPLB)\cite{clb}. OPLB is designed for contextual linear bandits, where each action involves both reward and cost signals. The key challenge is to construct a policy that balances exploration and exploitation under these constraints. The regret bound for OPLB is \( \tilde{O} \left( \frac{d\sqrt{T}}{\tau - c_0} \right) \), where \( \tau \) is the cost threshold and \( c_0 \) is the cost of a feasible action. This method provides a significant improvement in managing constrained settings effectively.

\textbf{Multi-agent Stochastic Bandits.} In recent years, there has been a growing interest in studying distributed or decentralized bandit problems. In the multi-armed bandit setting, \cite{DBLP:journals/corr/LandgrenSL16} investigated distributed multi-armed bandits, proposing two UCB-based algorithms, coopUCB and coop-UCB2. Coop-UCB assumes that all eigenvalues and corresponding eigenvectors of the structure matrix are known, which is a stronger assumption. In \cite{9143736}, a multi-agent multi-armed bandit setting is considered, incorporating communication costs between agents and proposing an efficient sampling rule and communication protocol to maximize each agent's expected cumulative reward. In \cite{ijcai2017p24}, each agent can either send information to the network or pull an arm. Another line of work addresses distributed MAB problems with collisions, where simultaneous arm selection by multiple agents results in no reward or split rewards \cite{landgren2020distributed, 5738217}. In \cite{Kar2011BanditPI}, only one agent can play an action at each time step and observe the corresponding reward, while others rely on the shared information. The closest setting to ours is in \cite{mamab}, where each agent has a different reward distribution for each arm, and the total regret is minimized relative to the global best arm, chosen by majority vote, which limits the action set to finite. We extend this to linear bandits (LB), handling both finite and infinite action sets. Unlike their setting, where actions and rewards are independent across players and arms, we model actions by feature vectors, allowing dependencies. For instance, in music recommendation systems, genres like rock and metal are not completely independent. Recent work \cite{dcsb} addresses similar MAB problems, with agents sharing information only with their neighbors, using Chebyshev acceleration for consensus \cite{Arioli2014}. Our paper adopts similar communication protocols for the algorithm description. In \cite{DBLP:journals/corr/KordaSL16}, agents share information with randomly selected agents rather than just neighbors. \cite{Gagrani2018ThompsonSF} examines two-agent team-learning problems under decoupled dynamics with no information sharing and coupled dynamics with delayed information sharing.

In this work, we present a novel approach to address the multi-agent constrained linear bandit problem, where a network of agents aims to optimize their cumulative rewards while maintaining the expected costs below a specified threshold. Each agent faces unique local bandit problems with distinct reward and cost parameters and communicates only with its immediate neighbors over the network. Our contribution includes the development of a new distributed algorithm, MA-OPLB, which leverages an accelerated consensus method to approximate average rewards and costs across the network. We derive high-probability regret bounds for our algorithm, demonstrating that it effectively balances exploration and exploitation under communication constraints. This work extends the theoretical understanding of multi-agent bandits and offers practical insights for applications requiring decentralized decision-making under constraints.
\section{Preliminaries}\label{sec:prelim}
Before introducing the main problem formulation and our results, we first introduce a set of definitions and lemmas.

\textbf{Notations.} 
For a positive integer \(n\), the set \(\{1, 2, \ldots, n\}\) is denoted by \([n]\).
For a vector \(x \in \mathbb{R}^d\) and
positive definite matrix $\Sigma \in \reals^{d\times d}$ we define
    \(\|x\|_{\Sigma,p} = \|\Sigma^{1/2} x\|_p\).
In particular, in the case of \(p = 2\), we have \(\|x\|_{\Sigma,2} = \sqrt{x^\top \Sigma x}\).
We denote the set of all probability distributions over any set \(\calD\)  by \(\Delta_{\calD}\).

\section{Problem Formulation}\label{sec:prob_formulation}

\textbf{Initial setup.} We study a multi-agent network comprising \(N\) agents that sequentially select actions influencing the entire network. This network is depicted as an undirected graph \(G\), with agents corresponding to nodes, denoted by the set \(\mathcal{V} = \{1, \dots, N\}\). We assume \(G\) is a connected graph, \emph{i.e.}, a path exists between every pair of nodes \(i\) and \(j\). Agents communicate solely with their adjacent nodes within this framework. We model interactions among agents relative to the graph \(G\), as a doubly stochastic matrix \(W\) with non-negative elements (\emph{i.e.}, \(W_{ij} \geq 0\) for all \(i, j\)), where \(W_{ij} = 0\) if and only if no edge connects node \(j\) to node \(i\). The presence of an edge is predefined by the graph's structure and considered an input to our model. The exact values of the non-zero entries in \(W\) are design parameters, which will be determined during the implementation of our bandit algorithms, as discussed in subsequent sections.

Each agent faces a unique local constrained linear bandit problem, characterized by distinct cost and reward parameters. Agents exchange information with their neighbors to address the overarching bandit problem collaboratively. Next, we describe the bandit problem that agents are solving, and then explain the network structure that we adopt.

\textbf{Local bandit problems.} Each agent \(i\) is presented with a local constrained linear bandit problem, defined by a reward parameter \(\theta^i_* \in \mathbb{R}^d\) and a cost parameter \(\mu^i_* \in \mathbb{R}^d\). In every round \(t\), agent \(i\) receives a set of possible actions \(\mathcal{D}_t \subset \mathbb{R}^d\) from which it must select an action \(x_t\). This action \(x_t\) does not carry an index \(i\) to indicate uniformity across the network, as will be elaborated later. When the agent chooses an action \(x_t \in \mathcal{D}_t\), it observes a reward signal \(r^i_t\) and a cost signal \(c^i_t\) given by 
\begin{equation}\label{eq:reward-and-cost}
r^i_t = x_t^\top \theta_*^i + \eta^i_{r,t}, \quad c^i_t = x_t^\top \mu^i_* + \eta^i_{c,t},  
\end{equation}
where \(\eta^i_{r,t}\) and \(\eta^i_{c,t}\) are random variables representing noise in reward and cost, respectively. These noise variables adhere to conditions to be detailed subsequently. Each agent maintains a policy \(\pi^i_t\), defined as a distribution over \(\mathcal{D}_t\), from which the agent can sample an action.

\textbf{The global bandit problem.} 
We consider the true global reward and cost parameters to be represented as 
\begin{equation}\label{eq:global_theta_mu}
    \theta_*^{\sf global} = \frac{1}{N}\sum_{i=1}^N \theta_*^i, \quad \mu_*^{\sf global} = \frac{1}{N} \sum_{i=1}^N \mu_*^i,
\end{equation}
%
At each round \(t\), the network coordinator randomly selects an agent index \(a(t) \in \mathcal{V}\), and all agents in the network execute the action proposed by the selected agent \(a(t)\), \emph{i.e.}, \(x_t \sim \pi_t^{a(t)}\). Each agent then observes its local reward and cost signals \(r_t^i, c_t^i\).

The agents' objective is to choose actions that maximize the network's global expected reward over \(T\) rounds, subject to a \textit{linear constraint} as expressed by
\begin{equation}\label{eq:global_constraint}
    \mathbb{E}_{x \sim \pi^{a(t)}_t}(x^\top \mu_*^{\sf global}) \leq \tau, \quad \forall t \in [T],
\end{equation}
where \(\tau > 0\) represents the known \textit{constraint threshold}. 

Given that each agent \(i\) only accesses observations based on its individual reward and cost parameters \(\theta^i_*, \mu^i_*\), effective decision-making to maximize the global expected reward and adhere to the global constraint necessitates collaboration among the agents. Consequently, agents must communicate to accurately estimate the true reward and cost parameters \(\theta_*^{\sf global}, \mu_*^{\sf global}\) in \eqref{eq:global_theta_mu}. 

The policy \(\pi^{a(t)}_t\), selected each round \(t \in [T]\) must reside within the feasible policy set defined over the action set \(\mathcal{D}_t\), denoted as
\begin{equation}
\Pi_t = \left\{ \pi \in \Delta_{\mathcal{D}_t} : \mathbb{E}_{x \sim \pi} \left( x^\top \mu_*^{\sf global} \right) \leq \tau \right\}.
\end{equation}
The optimization for maximum expected cumulative reward over \(T\) rounds translates into minimizing the \textit{constrained pseudo-regret},
\begin{equation}\label{eq:regret-def}
    \mathcal{R}(\theta_*^{\sf global}, T) = \sum_{t=1}^T \mathbb{E}_{x \sim \pi^*_t} \left( x^\top \theta_*^{\sf global} \right) - \mathbb{E}_{x \sim \pi^{a(t)}_t} \left( x^\top \theta_*^{\sf global} \right),
\end{equation}
where \(\pi^{a(t)}_t, \pi^*_t \in \Pi_t\) for all \(t \in [T]\). Here, \(\pi^*_t\) is the \textit{optimal feasible policy} for each round \(t\), formulated as
\begin{equation}
\pi^*_t = \max_{\pi \in \Pi_t} \mathbb{E}_{x \sim \pi_t} \left[ x^\top \theta_*^{\sf global} \right].
\end{equation}
This \(\pi^*_t\) represents the \textit{omniscient} feasible policy, achievable under full knowledge of the parameters \(\theta_*^{\sf global}\) and \(\mu_*^{\sf global}\), distinct from the best feasible policy attainable by an agent observing only noisy rewards and costs.

It is noteworthy that in our decentralized setup, while the network coordinator disseminates the selected action (\eg, a product) to the \(N\) agents (\eg, influencers), it does not receive any feedback from the social network nor performs computations, as this could be computationally intensive given the potentially large number of samples.

\textbf{Network structure.} The structure matrix \(W\) of the network is assumed to be symmetric and doubly stochastic, with each row and column summing to \(1\), ensuring \(1\) as an eigenvalue of \(W\). This leads to the eigenvalues of \(W\), satisfying the condition \(1 = \lambda_1 > |\lambda_2| \geq \cdots \geq |\lambda_N| \geq 0\), indicating that \(1\) is the largest eigenvalue in absolute value. Following the approaches of \cite{malb, dcsb}, our algorithm imposes information constraints on the agents regarding the graph's structure. Specifically, each agent is aware only of its direct neighbors, the total number of nodes in the network (\emph{i.e.}, agents), and the absolute value of the second largest eigenvalue, \(|\lambda_2|\).

The algorithm we propose utilizes decentralized communication among neighboring agents, facilitating the computation of average reward and cost values derived from the collective actions \(x_t^{a(t)}\) taken by the network. This mechanism enables agents to collaboratively estimate the average outcomes based on the decentralized interactions. The specifics of how agents determine the network action \(x_t^{a(t)}\) and the communication protocol for calculating the average reward, including the selection of non-zero values in the matrix \(W\), will be investigated in subsequent sections.

\textbf{Assumptions.} Before proceeding to the introduction of our algorithm, we state some standard assumptions.

\begin{assumption}\label{ass1}
    For all \(t\in T\), the reward and cost noise random variables \(\eta_{r, t}^i\), \(\eta_{c,t}^i\)
    are conditionally \(R\)-sub-Gaussian, 
    \begin{align*}
        &\Expect\left[\eta_{r, t}^i|\calF_{t-1}\right]=0, \quad \Expect\left[\exp\left(\alpha\eta_{r, t}^i\right)|\calF_{t-1}\right]\leq \exp\left(\alpha^2R^2/2\right), \\
        &\Expect\left[\eta_{c,t}^i|\calF_{t-1}\right]=0, \quad \Expect\left[\exp\left(\alpha\eta_{c,t}^i\right)|\calF_{t-1}\right]\leq \exp\left(\alpha^2R^2/2\right)
    \end{align*}
    for any \(\alpha\in\reals, i\in[N]\), where \(\calF_t\) is the filtration that includes all events \((x_{1:t+1}, \eta^i_{r,{1:t}}, \eta^i_{c,{1:t}}\) up to round \(t\).
\end{assumption}

\begin{assumption}\label{ass2}
    There is a known constant \(S > 0\), such that \(\|\theta^i_*\|\leq S\) and \(\|\mu^i_*\|\leq S\) for all \(i\in[N]\).
\end{assumption}

\begin{assumption}\label{ass3}
    The decision set $\calD_t$ is bounded. Specifically,
    \(\max_{t\in[T]}\max_{x\in\calD_t}\|x\|\leq L\).
\end{assumption}

\begin{assumption}\label{ass4}
    For all \(t\in [T]\) and \(x\in\calD_t\), the mean rewards and costs are bounded, \ie, \(x^\top\theta^{\sf global}_*  \in [0,1]\) and \(x^\top\mu^{\sf global}_* \in [0,1]\) for \(i\in[N]\).
\end{assumption}

\begin{assumption}\label{ass5}
    There exists a universally feasible action \(x_0 \in \calD_t\) for all \(t \in [T]\) associated with the cost \(c_0\). This means that \(x_0^\top \mu^{\sf global}_* = c_0 < \tau\). We assume that the value of \(c_0\) is known. Extending this to the cases where \(c_0 \) is unknown is straightforward. See \cite{clb} for further details on this scenario.
\end{assumption}

\section{Algorithm description}
To solve the networked constrained linear bandit problem, we introduce an enhanced algorithm based on the Optimistic-Pessimistic Linear Bandit (OPLB) framework proposed in \cite{clb}, named Multi-Agent OPLB (MA-OPLB). The pseudocode for MA-OPLB is given in Algorithm \ref{alg:multi_agent_oplb}. This algorithm operates episodically, with each episode divided into two distinct phases: 1) the exploration-exploitation phase and 2) the communication phase. During the exploration-exploitation phase, agents select and execute actions to optimize rewards while respecting the constraints. In the communication phase that follows, agents exclusively exchange information with their immediate neighbors in the network and repeatedly play the same action. By structuring the algorithm in this way, we can effectively balance action selection and communication needs. This design also allows us to measure the impact of prolonged communication phases, as the lack of estimation update during these periods contributes to an increase in regret.

\textbf{Exploration-exploitation phase.} Let \(t_s\) denote the start time of episode \(s\). At \(t_s\), the network coordinator randomly selects an agent index \(a(s)\). This agent \(a(s)\), utilizing a regularized least-squares estimate \(\hat{\theta}_{t_s}^{{\sf global}, a(s)}\) and \(\hat{\mu}_{t_s}^{{\sf global}, a(s)}\) for the global network parameters based on information up to \(t_s\), constructs ellipsoidal confidence regions \(\mathcal{C}_{r, t_s}^{a(s)}\) and \(\mathcal{C}_{c, t_s}^{a(s)}\), defined as 
\begin{align}\label{eq:conf_regions}
\begin{split}
    \mathcal{C}^i_{r, t} &= \left\{ \theta \in \mathbb{R}^d : \|\theta - \hat{\theta}^{{\sf global}, i}_{t}\|_{\Sigma_t} \leq \rho \beta_t \right\},\\
\mathcal{C}^i_{c, t} &= \left\{ \mu \in \mathbb{R}^d : \|\mu - \hat{\mu}^{{\sf global}, i}_{t}\|_{\Sigma_t} \leq \beta_t \right\},
\end{split}
\end{align}
where \( \rho=1+\frac{2}{\tau-c_0} \), \(\|.\|_{\Sigma_t, 2}\) is defined in Section \ref{sec:prelim}, and \(\beta_t\) is given by
\begin{equation}\label{eq:beta}
    \beta_t=\frac{R}{\sqrt{N}}\sqrt{d\log \frac{1+sL^2/\lambda}{\delta}}+\sqrt{\lambda}S+\frac{L}{\sqrt{\lambda}}.
\end{equation}

The agent then determines the optimal feasible policy and action by solving the constrained optimization problem
\begin{equation}\label{eq:optimal-policy}
    \left(\pi_{t_s}^{a(s)}, \tilde{\theta}_{t_s}^{a(s)}\right) = \argmax_{\pi \in \Pi^{a(s)}_{t_s}, \theta \in \mathcal{C}_{r, t_s}^{a(s)}} \mathbb{E}_{x \sim \pi}[x^\top \theta],
\end{equation}
and samples the network action according to the computed policy, \ie, \(x_{t_s}\sim \pi_t^{a(s)}\). Next, the chosen action is played by all the agents, and each agent observes their corresponding reward and cost, \(r_{t_s}^i\) and \(c_{t_s}^i\), according to \eqref{eq:reward-and-cost}. Next, MA-OPLB activates the communication phase.

\textbf{Communication phase.}
During the communication phase, agents do not engage in new action selection, and rather repeat the last selected network action. They instead focus on exchanging information with their neighbors to collectively approximate the average reward and cost signals observed during episode \(s\). Specifically, agents aim to estimate the average network reward \(\frac{1}{N} \sum_{i=1}^N r_{t_s}^i\) and the average network cost \(\frac{1}{N} \sum_{i=1}^N c_{t_s}^i\). Throughout this phase, agents forego selecting new actions, resulting in an increase in regret proportional to the duration of the communication phase. This highlights a trade-off between the precision of the agents' average reward and cost estimate and the regret incurred from extended communication periods. 

The communication protocol enables each agent to share its latest observed rewards \(r_{t_s}^i\) and costs \(c_{t_s}^i\) with its neighbors, following the network's structure matrix \(W\). Agents then receive their neighbors' signals and update their estimates using an accelerated consensus procedure as described in \cite{dcsb}. The details of this procedure, referred to as the mix function in line 12 of Algorithm \ref{alg:multi_agent_oplb}, are summarized in Algorithm \ref{alg:mix}.

Upon completing the communication phase of episode \(s\), each agent \(j\) has estimates \(r_{t_s}^i\) and \(c_{t_s}^i\), which approximates the average rewards and costs across all agents during the episode, \ie, \(\frac{1}{N} \sum_{i=1}^N r_{t_s}^i\) and \(\frac{1}{N} \sum_{i=1}^N c_{t_s}^i\). Using this approximation, agents update their regularized least-squares (RLS) estimates, given by
\begin{align}\label{eq:rls}
    \hat{\theta}^{{\sf global}, i}_{t}=\Sigma_t^{-1}\sum_{k=1}^s x_{t_k}y_k^i, \quad 
    \hat{\mu}^{{\sf global}, i}_{t}=\Sigma_t^{-1}\sum_{k=1}^s x_{t_k}z_k^i,
\end{align}
where $\Sigma_t = \lambda I + \sum_{k=1}^s x_{t_k}x_{t_k}^\top.$
The precision of these approximations is influenced by the length of the communication phase \(q(s)\), which we optimize to manage the trade-off between communication duration and regret. 

Lemma \ref{lemma:consensus_accuracy} from \cite{dcsb} provides a bound on the accuracy of the average reward approximation after a specified number of communication steps using the accelerated consensus procedure.

\begin{lemma}\label{lemma:consensus_accuracy}
    Let $W$ be a structure matrix with real eigenvalues such that $\mathbf{1}^\top W = \mathbf{1}^\top$ and $W \mathbf{1} = \mathbf{1}$ and all the eigenvalues are no more than one in absolute value. Fix $\epsilon > 0$, and let $q(\epsilon) = \left\lceil {\log(2N/\epsilon)}/{\sqrt{2 \log(1/|\lambda_2|)}} \right\rceil$ where $\lambda_2$ is the second largest eigenvalue in absolute value of matrix $W$. Then after $q(\epsilon)$ communication time steps based on Algorithm \ref{alg:mix}, each agent can construct a polynomial $p_{q(\epsilon)}(W)$ of the structure matrix $W$ that satisfies $\left\| p_{q(\epsilon)}(W) - \frac{1}{N} \mathbf{1} \mathbf{1}^\top \right\|_2 \leq \epsilon / N$.
\end{lemma} 

Additionally, define $r_s = [r^i_{t_s}]_{i=1,\ldots,N}$ and $c_s = [ 
 c^i_{t_s}]_{i=1,\ldots,N}$ as the vectors representing the rewards and costs observed by all agents at the start of episode $s$, prior to any communication. After $q(\epsilon)$ communication rounds, the agents utilize the mix function from Algorithm \ref{alg:mix} to apply the polynomial $p_{q(\epsilon)}(W)$ to these vectors, resulting in $p_{q(\epsilon)}(W) r_s$ and $p_{q(\epsilon)}(W) c_s$. The $j$-th components of these resulting vectors, denoted as $(p_{q(\epsilon)}(W) r_s)_j$ and $(p_{q(\epsilon)}(W) c_s)_j$, yield the estimates $y^j_s$ and $z^j_s$ for the average rewards and costs across the network, i.e., $\frac{1}{N} \sum_{i=1}^N  r^i_{t_s}$ and $\frac{1}{N} \sum_{i=1}^N  c^i_{t_s}$. 

Using a predefined $\epsilon$, Lemma \ref{lemma:consensus_accuracy} helps in bounding the approximation error for the reward estimate as
\small
\begin{align}
   \left\| p_{q(\epsilon)}(W) r_s - \frac{1}{N} \mathbf{1} \mathbf{1}^\top r_s \right\|_2 \leq \left\| p_{q(\epsilon)}(W) - \frac{1}{N} \mathbf{1} \mathbf{1}^\top \right\|_2 \| r_s \|_2 \leq \epsilon, \nonumber
\end{align} \normalsize
and similarly for the cost estimate. Here it is assumed that $\| r_s \|_2 \leq N$ and $\| c_s \|_2 \leq N$ as per Assumption \ref{ass4}. Each agent $j$ calculates the components $(p_{q(\epsilon)}(W) r_s)_j$ and $(p_{q(\epsilon)}(W) c_s)_j$ using only its own information and that from its neighbors. To compute the Chebyshev polynomials $p_{q(\epsilon)}(W)$, agents only require knowledge of the total number of agents $N$ and the second largest eigenvalue of $W$ in absolute value, $|\lambda_2|$ (details can be found in \cite{scaman2017optimal}).

\begin{algorithm}[H]
    \caption{Multi-agent OPLB}\label{alg:multi_agent_oplb}
    \begin{algorithmic}[1]
        \Require $\delta, T, \lambda$
        \State \(t\gets 1, s\gets 1\)
        \While{$t < T$}
            \State $q(s) \gets \left\lceil \log(2Ns) / \sqrt{2\log(1/|\lambda_2|)} \right\rceil$
            \State Set the start time of episode \(s\): $t_s \leftarrow t$
            \State The network coordinator selects the agent $\alpha(s)$
            \State Agent $a(s)$ computes policy
            \[
            \left(\pi_{t_s}^{a(s)}, \tilde{\theta}_{t_s}^{a(s)}\right) = \argmax_{\pi\in\Pi^{a(s)}_{t_s}, \theta\in\calC_{r,t_s}^{a(s)}} \Expect_{x\sim \pi}[x^\top \theta]
            \]
            \State Agent \(a(s)\) selects action \(x_{t_s}\sim \pi_t^{a(s)}\)
            \State Agents play the network action \(x_t\)
            \State Each agent \(i\) observes 
                \begin{align*}
                    r^i_{t_s} = x_{t_s}^\top \theta_*^i + \eta^i_{r,t_s}, ~ c^i_{t_s} = x_{t_s}^\top \mu^i_* + \eta^i_{c,t_s}
                \end{align*}

            \If{$t_s + q(s) > T$}
                \Return
            \Else
                \State Activate communication phase:
                \For{$h = 0$ to $q(s) - 1$}
                    \State $ r_{t_s}^i \leftarrow \text{mix}( r_{t_s}^i, h, i, [W_{ij}]_{j=1}^N)$ 
                    \State $c_{t_s}^i \leftarrow \text{mix}( c_{t_s}^i, h, i, [W_{ij}]_{j=1}^N)$
                    \State Play the last selected network action $x_{t_s}$
                \EndFor
            \EndIf
            \State \(y_s^i\gets r_{t_s}^i, \quad z_s^i\gets c_{t_s}^i\)
            \State $t \leftarrow t_s + q(s)$
            \State Update the RLS-estimates \(\hat{\theta}^{{\sf global}, i}_{t}\) and \(\hat{\mu}^{{\sf global}, i}_{t}\)
            \State Build the confidence regions \(\calC^i_{r, t},\calC^i_{c, t}\) according to \eqref{eq:conf_regions}
            \State $s \leftarrow s + 1$
        \EndWhile
    \end{algorithmic}
\end{algorithm}

\begin{algorithm}[H]
    \caption{\(\text{mix}(\alpha_h^i, h, i, [W_{ij}]_{j=1}^N)\)}\label{alg:mix}
    \begin{algorithmic}[1]
    \Require \(h,\alpha_h^i, \alpha_{h-1}^i, \lambda_2, i, [W_{ij}]_{j=1}^N \)
        \If{$h = 0$}
            \State $c_0 \gets 1/2$, $c_{-1} \gets 0$
            \State $\alpha_0^i \gets \alpha_0^i / 2$, $\alpha_{-1}^i \gets (0, \ldots, 0)$
        \EndIf
        \State send $\alpha_h^i$ to the neighbors, and receive their corresponding values $\alpha_h^j$, $\forall j \in \mathcal{N}(i)$
        \State Compute: $z_h^i = \sum_{j \in \mathcal{N}(i)} 2W_{i,j} \alpha_h^j / |\lambda_2|$
        \State Update: $c_{h+1} \gets 2c_h / |\lambda_2| - c_{h-1}$
        \State Compute: $\alpha_{h+1}^i = \frac{c_h}{c_{h+1}} z_h^i - \frac{c_{h-1}}{c_{h+1}} \alpha_{h-1}^i$
        \If{$h = 0$}
            \State $c_0 \gets 2c_0$, $\alpha_0^i \gets 2\alpha_0^i$
        \EndIf
        \State \Return $\alpha_{h+1}^i$
    \end{algorithmic}
\end{algorithm}

\section{Regret Analysis}
Recall that each agent \(i\) operates with its own reward and cost parameters \(\theta^i_*\) and \(\mu^i_*\), respectively. Nevertheless, the objective is to optimize the global network reward characterized by the parameters \(\theta_*^{\sf global}\) and \(\mu_*^{\sf global}\), which are the averages of the agents' individual parameters. Each agent must estimate the global parameters, yet communication constraints allow them to exchange information only with their immediate neighbors. Our approach involves periodically incorporating communication rounds into the algorithm, enabling agents to disseminate their gathered data. Specifically, after receiving their reward and cost signals, agents engage in an accelerated consensus mechanism to share information. This process permits agents to approximate the average network reward and cost. Subsequently, in Theorem \ref{thm:conf-sets} (Theorem IV.1 of \cite{malb}), it is demonstrated that using these approximations, each agent can construct a revised confidence ellipsoid that probabilistically contains the true global parameters \(\theta_*^{\sf global}\) and \(\mu_*^{\sf global}\).

\begin{theorem}\label{thm:conf-sets}
    Let Assumptions \ref{ass1}, \ref{ass2}, \ref{ass3}, and \ref{ass4} hold. Fix any \(\delta \in (0,1)\), and let the structure matrix \(W\) satisfy the conditions required in Lemma \ref{lemma:consensus_accuracy}. Furthermore, let \(t_s\) be the start time of episode \(s\). We choose the length of the communication phase as \(q(s) = \left\lceil \log(2Ns) / \sqrt{2\log(1/|\lambda_2|)} \right\rceil\). Then, at each time \(t\) in the interval \([t_s, t_{s+1})\) and using the information up to round \(t_s\), each agent \(i\) can construct confidence regions \(\calC^i_{r, t}\) and \(\calC^i_{c, t}\) according to \eqref{eq:conf_regions}. Then, the parameters \(\theta_*^{\sf global}\) and \(\mu_*^{\sf global}\) are included in the confidence regions \(\calC^i_{r, t}\) and \(\calC^i_{c, t}\) with probability at least \(1-\delta\), respectively, \ie,
    \[
    \Prob(\theta_*^{\sf global}\in \calC^i_{r, t}) \geq 1 - \delta, \quad \Prob(\mu_*^{\sf global}\in \calC^i_{c, t}) \geq 1 - \delta.
    \]
\end{theorem}
\begin{proof}
    We explain the proof for the reward parameter. The proof for the cost parameter is similar. First, recall that agents only update the RLS-estimates when the communication phase finishes, and they have computed an approximation of the average of the rewards and costs according to \eqref{eq:rls}. Then, from Lemma \ref{lemma:consensus_accuracy} we know that after the communication phase of each episode $s$, we have $\left| y_s^i - \frac{1}{N} \sum_{j=1}^N r_{t_s}^j \right| \leq \epsilon$.
Therefore, we can write
\[
y_s^i =  \frac{1}{N} \sum_{j=1}^N r_{t_s}^j + \gamma \quad \text{such that} \quad |\gamma| \leq \epsilon.
\]
Equivalently, we can get:
\begin{align*}
    &\hat{\theta}^{{\sf global}, i}_s = \Sigma_t^{-1} \sum_{k=1}^s x_{t_k} \left( \frac{1}{N} \sum_{j=1}^N r_{t_k}^j + \gamma \right)\\
    ~~~~&= \Sigma_t^{-1} \sum_{k=1}^s x_{t_k}\frac{1}{N}\sum_{j=1}^N \left( \left\langle x_{t_k}, \theta_*^j \right\rangle + \eta_{r,t_k}^j \right) + \Sigma_t^{-1} \sum_{k=1}^s x_{t_k} \gamma\\
    ~~~~&= \Sigma_t^{-1} \sum_{k=1}^s x_{t_k} \left( \left\langle x_{t_k}, \theta^{\sf global}_* \right\rangle + \zeta_{t_k} \right) + \Sigma_t^{-1} \sum_{k=1}^s x_{t_k} \gamma,    
\end{align*}
where \(\zeta_{t_k} = \frac{1}{N} \sum_{i=1}^N \eta_{r,t_k}^i\) is conditionally $\frac{R}{\sqrt{N}}$-sub-Gaussian noise. Then, following the machinery in \cite{abbasi}, we can write
\small
\begin{align}
    \left\| \hat{\theta}^{{\sf global}, i}_s - \theta^{\sf global}_* \right\|_{\Sigma_t} \leq &
\lambda \left\| \theta^{\sf global}_* \right\|_{\Sigma_t^{-1}} 
+ \left\| \sum_{k=1}^s x_{t_k} \zeta_{t_k} \right\|_{\Sigma_t^{-1}} \nonumber \\ &
+ \gamma \left\| \sum_{k=1}^s x_{t_k} \right\|_{\Sigma_t^{-1}}.
\end{align}
\normalsize
We can bound $\left\| \sum_{k=1}^s x_{t_k} \zeta_{t_k} \right\|_{\Sigma_t^{-1}}$ by Theorem 1 in \cite{abbasi}. Also, according to Assumption \ref{ass3}, we can have $\gamma \left\| \sum_{k=1}^s x_{t_k} \right\|_{\Sigma_t^{-1}} \leq \gamma s L / \sqrt{\lambda}$. 
If we choose $\gamma = 1/s$, with probability at least $1 - \delta$, we can get $\left\| \hat{\theta}^{{\sf global}, i}_s - \theta^{\sf global}_* \right\|_{\Sigma_t} \leq \beta_t$, where \(\beta_t\) is given in \eqref{eq:beta}, and the proof is complete.
\end{proof}
Consequently, we demonstrate that agents can build confidence regions that include the true global reward and cost parameters, \(\theta_*^{\sf global}\) and \(\mu_*^{\sf global}\), with high probability using approximate average rewards and costs. This approach, however, comes with an additional cost, specifically \(\sqrt{L}/\lambda\) as outlined in equation \eqref{eq:beta}. Note that this \(\lambda\) is the regularization parameter in least squares and is different from eigenvalues of the structure matrix \(W\). With this foundation, we are now ready to bound the overall regret of the network. In Theorem \ref{thm:regret_bound}, we provide a regret bound for the Multi-agent OPLB algorithm.

\begin{theorem}\label{thm:regret_bound}
     Let Assumptions Assumptions \ref{ass1}, \ref{ass2}, \ref{ass3}, and \ref{ass4} hold, $\lambda\geq \max(1, L^2)$,
and the structure matrix $W$ satisfy the condition mentioned in Lemma \ref{lemma:consensus_accuracy}. For a fixed $\delta\in(0,1)$, the regret defined in \eqref{eq:regret-def} for Algorithm \ref{alg:multi_agent_oplb} is upper bounded with probability at least $1-\delta$ by
\begin{align}\label{eq:maoplb-regret}
\small
    \begin{split}
        \mathcal{R}(\theta_*^{\sf global}, T) \leq& \Bigg(\frac{2L(\rho+1)\beta_{T'}}{\sqrt{\lambda}}\sqrt{2T'\log\frac{1}{\delta}}  \\ 
        &+ (\rho+1)\beta_{T'}\sqrt{2T'd\log\left(1+\frac{T'L^2}{\lambda}\right)}+1\Bigg)\\
        &\times \left(1+\frac{\log(2Ns)}{\sqrt{2\log(1/|\lambda_2|)}}\right), \nonumber
    \end{split}
    \normalsize
\end{align}
where \(T'=\frac{T}{1+q(1)}\), \(q(1)=\frac{\log(2N)}{2\log(1/|\lambda_2|)}\).
\end{theorem}

\begin{proof}
    The key idea is that the overall regret of
the network can be decomposed into two terms: regret caused
by the exploration-exploitation phases, and that of communication phases. We know that during the communication phase, agents share information and they only repeat the last selected network action. Recall that according to the Theorem \ref{thm:conf-sets}, the length of the communication phase
$s$ is $q(s) = \left\lceil \log(2Ns) / \sqrt{2\log(1/|\lambda_2|)} \right\rceil$. Therefore, the number of communication phases up to round $T$ is upper bounded by $T/q(1)$ (because the length of communication phases is increasing in time). We upper bound the regret defined in \eqref{eq:regret-def} up to round $T$ with the regret of a linear bandit problem with $T' = T/q(1)$ time steps multiplied with
maximum regret of communication phases. From Assumption \ref{ass2}, \ref{ass3}, and \ref{ass4}, we can bound the regret caused by the communication phases with 
$2\left\lceil \log(2Ns) / \sqrt{2\log(1/|\lambda_2|)} \right\rceil$. Then, from Theorem 2 in \cite{clb}, we can bound the regret of a constrained linear bandit problem with \(T'\) time steps with probability \(1-\delta \) as 
\begin{align}
    \frac{2L(\rho+1)\beta_{T'}}{\sqrt{\lambda}} & \sqrt{2T'\log\left(1/\delta\right)}  \nonumber \\& + (\rho+1)\beta_{T'}\sqrt{2T'd\log\left(1+{T'L^2}/{\lambda}\right)}.
\end{align}
Putting these arguments together completes the proof.
\end{proof}
\begin{remark}
    The upper bound provided for the regret of Algorithm \ref{alg:multi_agent_oplb} over a period of length \(T\) has the order of 
    \begin{align}\label{eq:bound-order}
        \mathcal{O}\left(\frac{d}{\tau-c_0}\frac{\log(NT)}{\sqrt{N}}\log\left(\frac{TL^2}{\delta}\right)\sqrt{\frac{T}{\log(1/|\lambda_2|)}}\right).
    \end{align}
\end{remark}
%
%
%
Equation \eqref{eq:bound-order} shows how the regret bound depends on different components of the problem. Two important elements are the spectral gap of the structure matrix, defined as $SG(W)=1-|\lambda_2|$, and the cost gap of the known feasible action \(\tau-c_0\). Networks with larger spectral gaps and safe actions with larger cost gaps will result in a smaller regret bound and a better performance of our algorithm.
%
%
%
\begin{figure*}
    \centering
    \begin{subfigure}[b]{0.4\textwidth}
        \centering
        \includegraphics[width=0.8\textwidth]{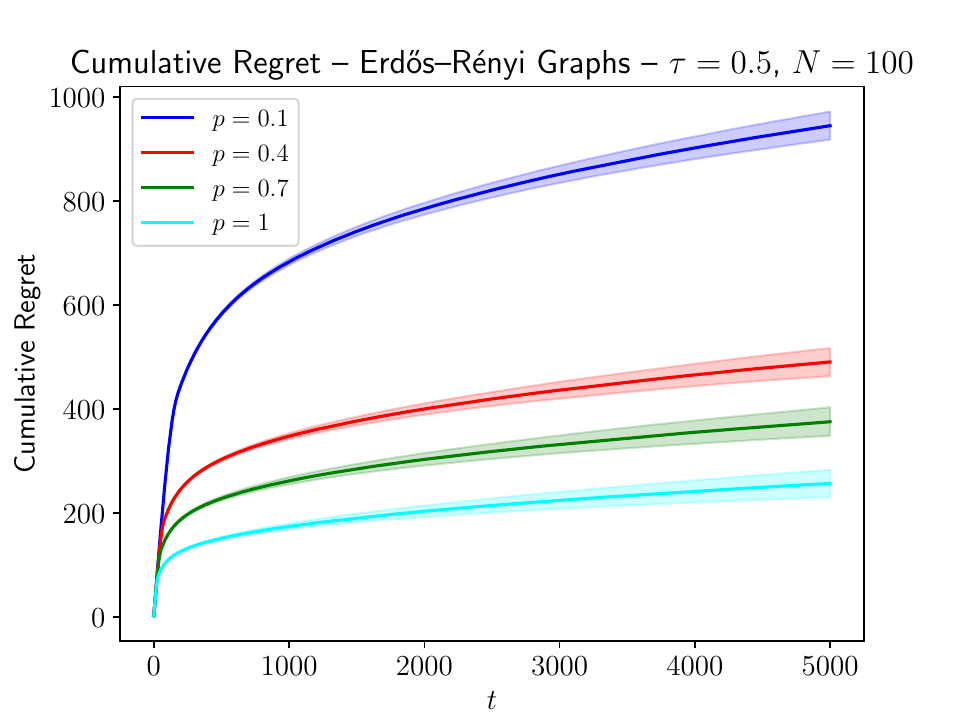}
        \vspace{-.2cm}
        \caption{Cumulative regret for Erdős–Rényi graphs}
        \label{fig:regret-erdos-renyi}
    \end{subfigure}
    \begin{subfigure}[b]{0.4\textwidth}
        \centering
        \includegraphics[width=0.8\textwidth]{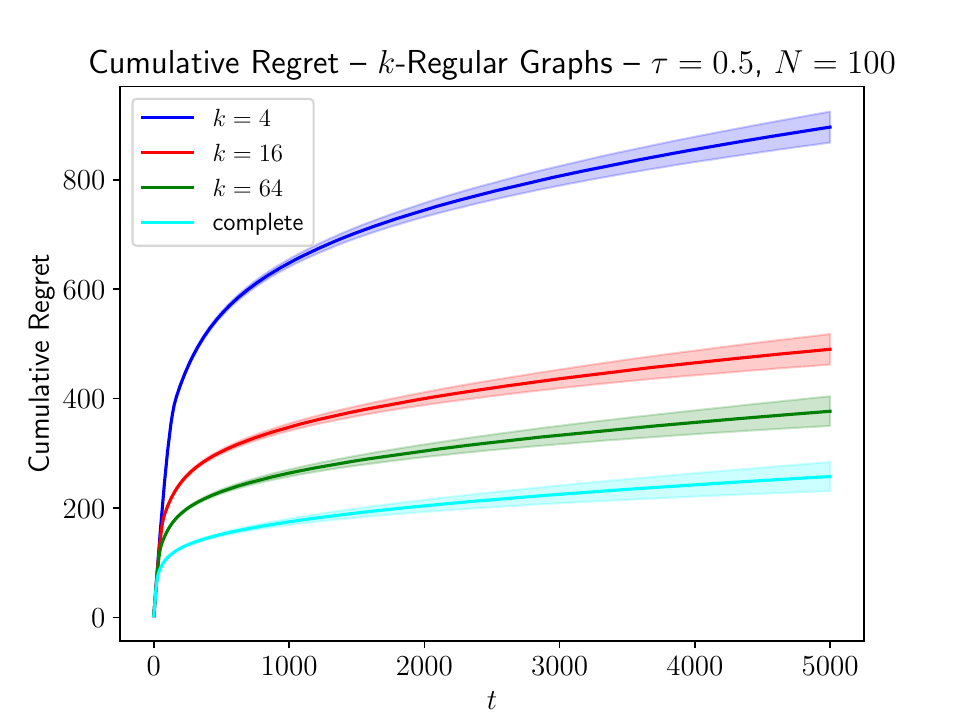}
        \vspace{-.2cm}
        \caption{Cumulative regret for $k$-regular graphs}
        \label{fig:regret-k-regular}
    \end{subfigure}
    \begin{subfigure}[b]{0.4\textwidth}
        \centering
        \includegraphics[width=0.8\textwidth]{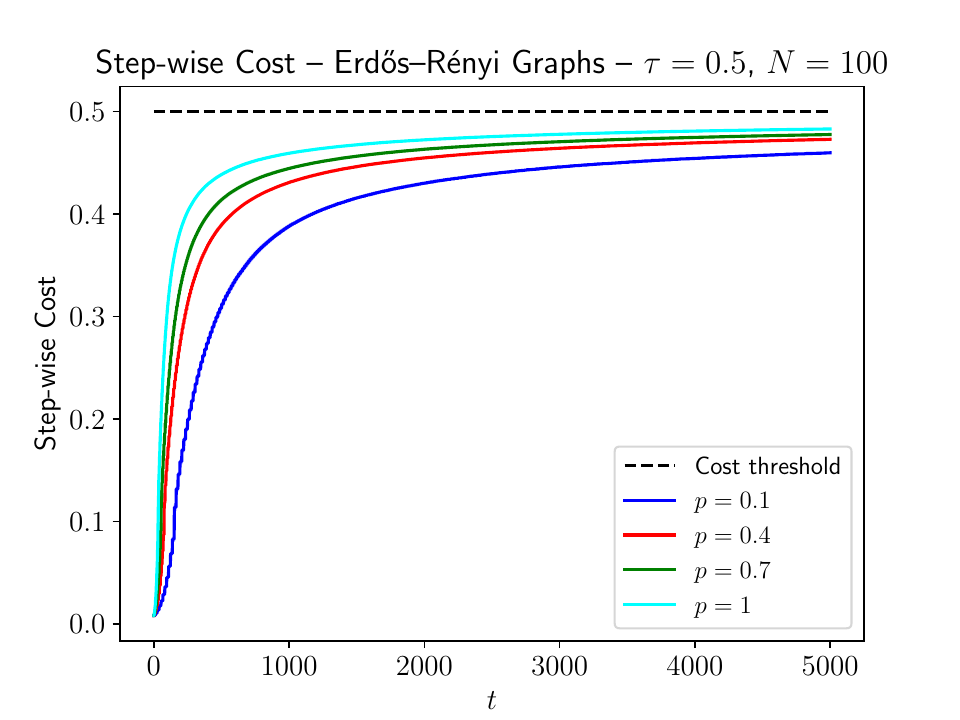}
        \vspace{-.2cm}
        \caption{Step-wise cost for Erdős–Rényi graphs}
        \label{fig:cost-erdos-renyi}
    \end{subfigure}
    \begin{subfigure}[b]{0.4\textwidth}
        \centering
        \includegraphics[width=0.8\textwidth]{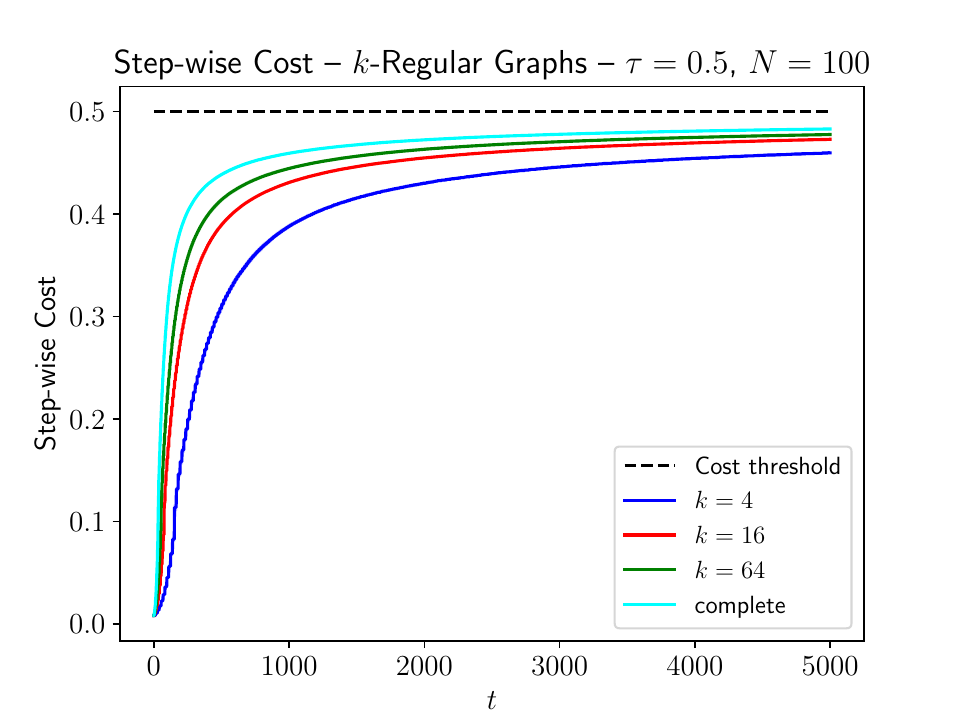}
        \vspace{-.2cm}
        \caption{Step-wise cost for $k$-regular graphs}
        \label{fig:cost-k-regular}
    \end{subfigure}
    \caption{Comparison of cumulative regret and step-wise cost for 
    Erdős–Rényi and $k$-regular graph structures}
    \label{fig:regret-cost}
    \vspace{-.5cm}
\end{figure*}
%
%
\section{Experiments}
In this section, we evaluate the performance of the proposed MA-OPLB algorithm across various network structures and parameters. Our experiments are designed to demonstrate how different graph topologies and network sizes influence the cumulative regret and the algorithm's effectiveness in adhering to the prescribed constraints. The core optimization problem in MA-OPLB (line 6 of Algorithm \ref{alg:multi_agent_oplb}) is inherently challenging due to its non-linear and non-convex nature, and currently, there are no established methods that guarantee convergence to the global minimizer. To address this complexity, we employ the convex programming tools with \(\ell_1\) relaxation methods introduced in \cite{afsharrad2024convex} for our simulations.

\textbf{Construction of the structure matrix.} Given a network topology represented by an undirected graph \( G \), we determine the locations of the zero and non-zero elements of \( W \) from the adjacency matrix \( A \) of \( G \). Then, we need to construct the structure matrix \( W \) as a symmetric, doubly stochastic matrix. To assign numerical values to the non-zero elements of \( W \) and ensure it is doubly stochastic, we follow the approach used in \cite{duchi2012dual}. Specifically, for non-\( k \)-regular graphs with maximum node degree \( \delta_{\max} \), we construct \( W \) using the normalized Laplacian matrix \( \mathcal{L} = I - D^{-1/2} A D^{-1/2}\), where \( D \) is a diagonal matrix with \( D_{ii} \) being the degree of node \( i \). We then define the structure matrix \( W \) as \(W = I - \frac{1}{\delta_{\max} + 1} D^{1/2} \mathcal{L} D^{1/2}\). In the case of \( k \)-regular graphs, we define \( W = I - \frac{k}{k + 1} \mathcal{L}\).

\textbf{Simulation parameters.} For the numerical simulations, unless otherwise stated, we use the following parameters: \(\tau=0.5\), \(c_0=0\), \(d=2\), \(T=5000\), \(N=100\), \(\lambda=0.1\), \(R=0.05\), \(\delta=0.01\), and the decision set at every time step \(\mathcal{D}_t = \{x \in \mathbb{R}^2 \mid \|x - (0.5, 0.5)\| \leq 1/\sqrt{2}\}\). For the global reward and cost parameters, we set \(\theta_*^{\sf global} = \begin{bmatrix}2.7 & 0.6 \end{bmatrix}^\top\) and \(\mu_*^{\sf global} = \begin{bmatrix}0.1 & 0.8 \end{bmatrix}^\top\). The norm of \(\theta_*^{\sf global}\) is intentionally chosen to be large for better visualization. The local reward and cost parameters are randomly selected.

\subsection{Performance over Different Network Structures}

We first analyze the impact of different network structures on the cumulative regret and cost by considering two types of graphs: Erdős-Rényi random graphs and $k$-regular graphs. In both cases, we set the number of agents to $N=100$ and vary the parameters that define the graph topology.

\subsubsection{Erdős-Rényi Random Graphs}

An Erdős-Rényi (ER) graph \(G(n, p)\) is a random graph where each possible edge between \(n\) nodes is included with independent probability \(p\). The parameter \(p\) controls the density of the graph: a larger \(p\) results in a more connected graph, hence a smaller \(|\lambda_2|\). In our experiments, we consider \(p \in \{0.1, 0.4, 0.7, 1.0\}\), where \(p=1.0\) corresponds to a complete graph where every pair of nodes is connected. We measure the cumulative regret and the step-wise cost incurred by the MA-OPLB algorithm on these graphs. The results, presented in Figures~\ref{fig:regret-erdos-renyi} and~\ref{fig:cost-erdos-renyi}, show how the connectivity of the network affects the algorithm's performance. As expected, higher connectivity (larger \(p\)) leads to better information dissemination among agents, resulting in lower cumulative regret.

By examining the cost plots in Figure~\ref{fig:cost-erdos-renyi}, we observe that the algorithm consistently maintains the cost below the threshold \(\tau\) for all graph configurations. Moreover, in graphs with higher connectivity (larger \(p\)), the cost approaches the threshold faster. This behavior occurs because increased connectivity allows agents to learn the environment more efficiently, enabling them to maximize their collected reward by pushing the cost up to the allowable limit. In contrast, in graphs with lower connectivity (smaller \(p\)), it takes longer for the agents to build confidence in their estimates of the global parameters. Consequently, the algorithm remains more conservative for an extended period, keeping the cost further below the threshold until sufficient confidence is achieved.

\subsubsection{\(k\)-Regular Graphs}

A \(k\)-regular graph is a graph where every node has exactly \(k\) neighbors. Regular graphs are valuable for studying the effect of uniform connectivity constraints on the network's performance. We construct \(k\)-regular graphs with \(k \in \{4, 16, 64, 99\}\), where \(k=99\) corresponds to a complete graph with \(N=100\) nodes. Similar to the ER graphs, we run the MA-OPLB algorithm on these \(k\)-regular graphs and record the cumulative regret and step-wise cost. The results, depicted in Figures~\ref{fig:regret-k-regular} and ~\ref{fig:cost-k-regular}, indicate that the cumulative regret decreases as the regularity degree \(k\) increases, confirming that more connectivity and a smaller \(|\lambda_2|\) results in a smaller regret. Moreover, by inspecting the step-wise cost graph, we observe behavior analogous to that in the Erdős-Rényi graphs. The algorithm maintains the cost below the threshold \(\tau\) in all cases, and as the regularity degree \(k\) increases (i.e., higher connectivity), the cost approaches the threshold more rapidly. This can be attributed to the same reasons explained earlier: higher connectivity enables agents to learn the environment more quickly, allowing them to confidently push the cost up to the limit to maximize the collected reward. 
\subsection{Effect of Network Size}
We now investigate the impact of the number of agents \( N \) on the performance of the MA-OPLB algorithm. According to our theoretical regret bound, the regret scales with \( \frac{\log N}{\sqrt{N}} \), suggesting that as \( N \) increases, the overall regret should \emph{decrease}. This improvement arises because, although the communication overhead grows proportionally to \( \log N \), the reduction in noise achieved by averaging across more agents--scaling with \( \frac{1}{\sqrt{N}} \)--has a more substantial effect. Consequently, employing a larger number of agents enhances the algorithm's performance. To better observe the noise reduction effect, we choose the noise level parameter \(R=0.5\) for the numerical simulations.

To isolate the effect of \( N \) and ensure that variations in the second largest eigenvalue \( |\lambda_2| \) do not influence the results, we conduct our experiments on \emph{complete graphs} for different values of \( N \). In complete graphs, all pairs of nodes are connected, and the second largest eigenvalue remains constant (specifically, \( |\lambda_2| = 0 \)). This allows us to attribute any observed changes in performance solely to the variation in \( N \).

In our experiments, we consider \( N \in \{4, 16, 64\} \) agents and, instead of measureing the cumulative regret, we use a more direct measure of the algorithm's learning progress, the cumulative estimation error in the global reward parameter, given by
\(\sum_{t=1}^T \left\| \frac{1}{N} \sum_{i=1}^N \hat{\theta}^{\sf global, i}_t - \theta_*^{\sf global} \right\|.\)
This measure represents the cumulative norm of the difference between the average estimated global reward parameter and the true parameter. This metric provides a clearer visualization of how quickly the agents collectively learn the true parameter as \( N \) varies.
\begin{figure}[h!]
    \centering
    \includegraphics[width=.35\textwidth]{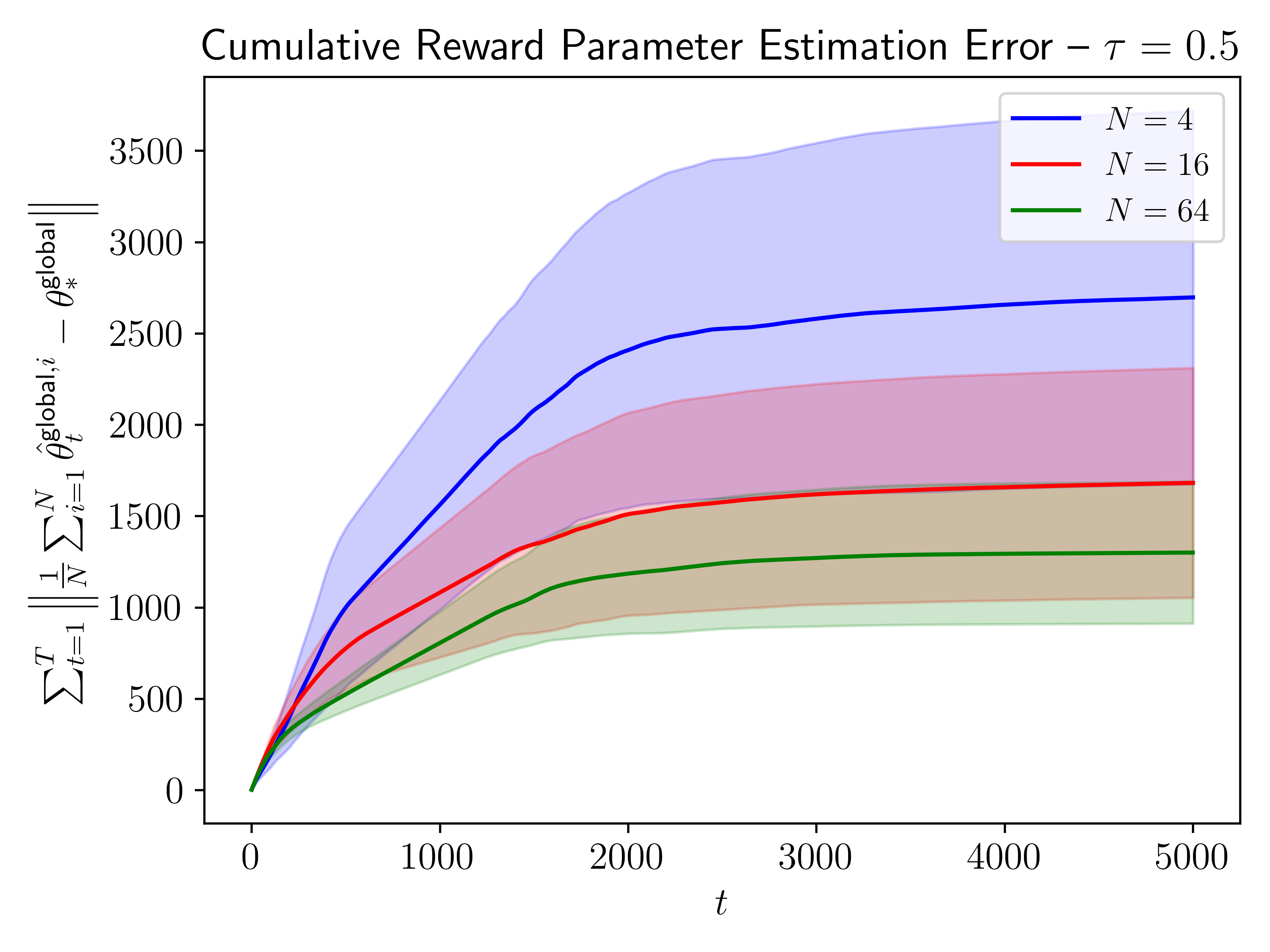}
    \vspace{-.5cm}
    \caption{Cumulative reward parameter estimation error 
    }
    \label{fig:cumulative-parameter-error}
\end{figure}
The results, presented in Figure~\ref{fig:cumulative-parameter-error}, show that the cumulative estimation error decreases as \( N \) increases. This confirms that larger networks allow agents to average out noise more effectively, leading to faster and more accurate estimation of the global reward parameter.
\subsection{Convergence Behavior}
Finally, we study the convergence behavior of the MA-OPLB algorithm by examining the trajectory of the selected actions over time. In a 2D action space, we plot the actions chosen by the agents at each time step to visualize how they converge to the optimal action while satisfying the global cost constraint.
\begin{figure}[h!]
    \centering
    \includegraphics[width=.38\textwidth]{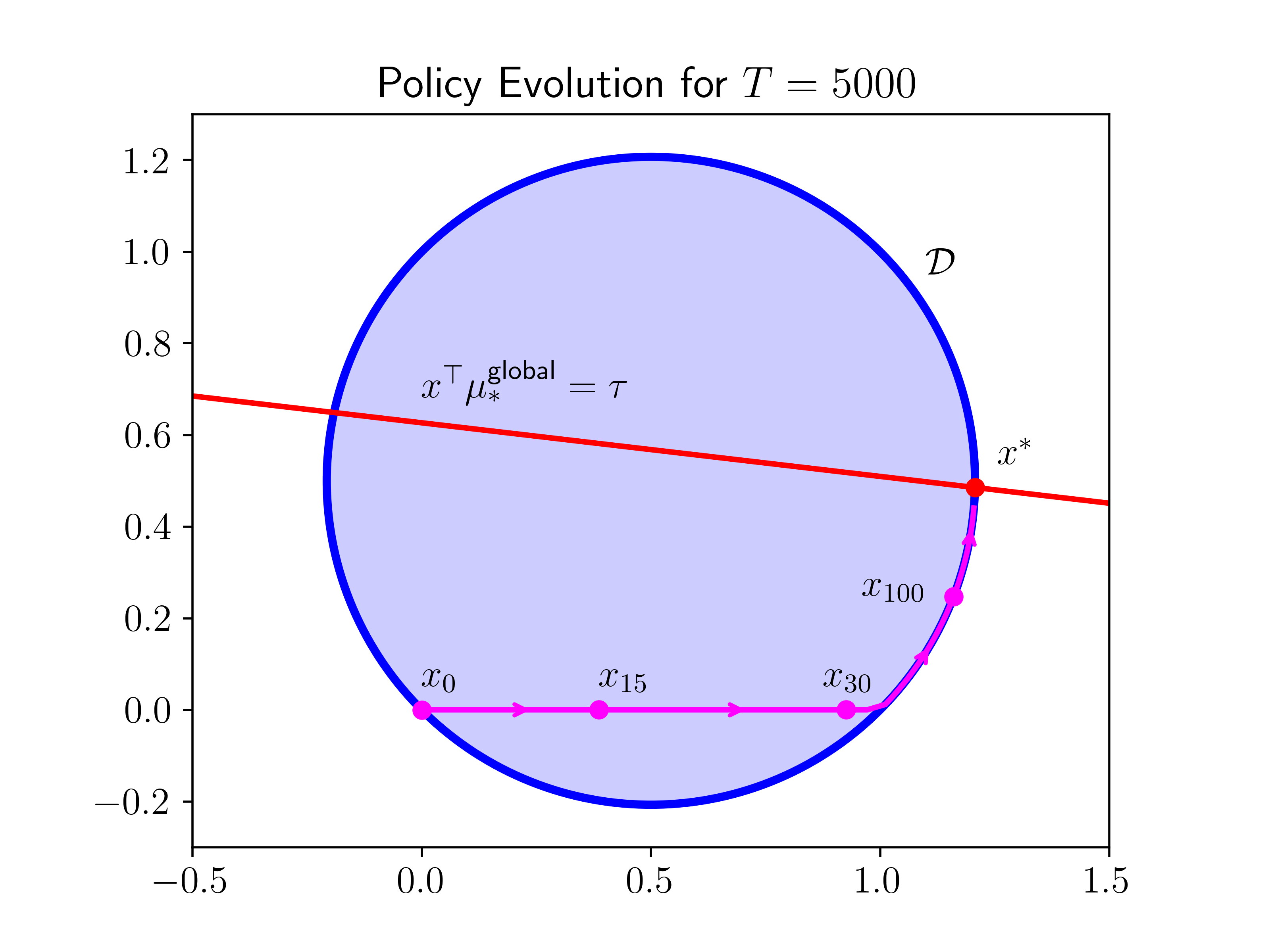}
    \vspace{-.5cm}
    \caption{Trajectory of actions over time in a 2D action space}
    \label{fig:action-trajectory}
\end{figure}
Figure~\ref{fig:action-trajectory} illustrates the trajectory of the actions. The plot shows that the actions progressively move towards the optimal action, confirming the algorithm's ability to learn and adapt over time. Moreover, the actions remain within the feasible region defined by the cost constraint, demonstrating that the algorithm effectively ensures constraint satisfaction throughout the learning process.
\section{Conclusion and future work}
We studied the problem of multi-agent constrained stochastic linear bandits, where agents aim to maximize cumulative rewards under a global cost constraint. Each agent interacts with its own local bandit problem with unique reward and cost parameters, and the collective goal is to optimize according to the global parameters, which are averages of the local ones. We proposed the MA-OPLB algorithm, which combines exploration and exploitation with an accelerated consensus method, allowing agents to estimate the global parameters through local communications with their neighbors. We derived a frequentist regret bound for MA-OPLB, showing that the regret scales with the spectral gap of the network's communication matrix and inversely with the cost gap of a known feasible action.
%
Future work could extend this framework to scenarios where agents are not required to play the same action, allowing for more flexibility and potentially improved performance in heterogeneous environments.

\bibliographystyle{IEEEtran}
\bibliography{refs.bib}
\end{document}